\documentclass{llncs}
\usepackage{cite}
\usepackage{amsmath,amssymb,amsfonts}
\usepackage{mathtools}
\usepackage{aligned-overset}
\usepackage{bbm}
\usepackage{algorithm,algpseudocode}
\usepackage{algpseudocode}
\usepackage{graphicx}
\usepackage{textcomp}
\usepackage{xcolor}
\usepackage{booktabs}
\usepackage{siunitx}
\DeclareSIUnit\pixel{px}
\usepackage{lipsum}

\usepackage{multirow}

\usepackage[normalem]{ulem}

\title{Rotation Invariance in Floor Plan Digitization using Zernike Moments}

\begin{document}
\author{Marius, Graumann\inst{1,2}\orcidID{0009-0005-1203-8414} \and Jan Marius St\"{u}rmer\inst{1,3}\orcidID{0009-0002-1490-6607}, Tobias Koch\inst{1,4}\orcidID{0000-0003-1279-0209}}
\institute{German Aerospace Center (DLR), Institute for the Protection of Terrestrial Infrastructures \and \email{marius.graumann@dlr.de} \and
\email{jan.stuermer@dlr.de}\and
\email{tobias.koch@dlr.de}}

\maketitle
\begin{abstract}
Nowadays, a lot of old floor plans exist in printed form or are stored as scanned raster images. Slight rotations or shifts may occur during scanning. Bringing floor plans of this form into a machine readable form to enable further use, still poses a problem.
 Therefore, we propose an end-to-end pipeline that pre-processes the image and leverages a novel approach to create a region adjacency graph (RAG) from the pre-processed image and predict its nodes. By incorporating normalization steps into the RAG feature extraction, we significantly improved the rotation invariance of the RAG feature calculation. Moreover, applying our method leads to an improved F1 score and IoU on rotated data. Furthermore, we proposed a wall splitting algorithm for partitioning walls into segments associated with the corresponding rooms.
 \and indoor spatial data

\end{abstract}
\section{Introduction}
\label{sec:introduction} 
A floor plan is a drawing that shows the shape, size, and arrangement of rooms in a building as viewed from above. 
Floor plans consist of structural indoor elements, such as walls, windows, doors, stairs, and spatial elements, like rooms and corridors. 
Structural and spatial elements in floor plans play a crucial role while designing, understanding, or remodelling indoor spaces \cite{Liu.2017}, when simulating pedestrians movements and when creating 3D models \cite{Barreiro.2023}\cite{Deshmukh.2023}. 
Floor plans are often generated using computer-aided design (CAD), but are frequently stored as raster images or in printed form\cite{Yang.2018}. During this conversion the images become blurred or rotated, which poses a problem for complex architectural drawings \cite{Kim.2021}. Moreover, the floor plans are no longer machine readable.

Floor plan classification is a difficult task, as there is no standard notation in architectural and engineering companies. Therefore, coloring, line width and used symbols differ \cite{Mace.2010}. 
Furthermore, the classification 
should suffice higher-level geometric and topological conditions, i.e., doors are embedded in walls and walls delimit rooms. 
Moreover, the room layout can depend on the use case, e.g., apartments or office rooms can have different room layouts \cite{Liu.2017}. 
Therefore, learning based approaches are most promising.

Another challenge are complex building layouts. Office buildings, schools, exhibition buildings for example can have complex structures such that some walls are not vertical, horizontal or straight. 
Popular floor plan training data consists mostly of data depicting apartments with nicely behaved structures, e.g., walls are vertical or horizontal. To handle more complex structures that are not represented inside the training set, the classification analysis has to be rotation invariant.
Moreover, rotation invariance is essential, as it enables scanned floor plans to be analyzed more reliably due to possible rotation from imprecise scanning.

In this paper, we present an enhanced method based on the work in \cite{JaeyoungSong.2021} to classify floor plan components and further extract relations, such as room-door connections in a RCG. We explain how certain attributes (moments) are adjusted to achieve translation invariance.  

The paper is structured as follows: First, we introduce current state-of-the-art approaches in section II. Second, we introduce our end-to-end-pipeline consisting of the enhanced feature extraction method (section III).
In sections IV and V, we elaborate on the experimental setup and show the results. Next, we discuss the results and further benefits or drawbacks in section VI, and draw our conclusion in section VII. 

\section{Related work}
\label{sec: related work}
Several approaches towards the classification and analysis of floor plans exist. Therefore, we give an overview about these approaches and then concentrate on work that treats rotation invariant methods based on Zernike moments.
\subsection{Floor plan classification}
To achieve floor plan reconstruction, different methods might be applied. Machine learning based approaches are promising for handling the wide spectrum of different floor plans and perform better in comparison to rule based methods \cite{Summary}.

The task of converting 2D floor plans to semantic 3D models was addressed by \cite{Barreiro.2023}. Their approach utilized Faster-RCNN with a ResNet backbone for window and door detection, resulting in bounding boxes. For wall detection, they employed the FPN architecture paired with a ResNet backbone, which produced segmentation masks for walls. The obtained segmentations and bounding boxes underwent post-processing to vectorize walls, doors, and windows, which were then used to construct the 3D model. On the CubiCasa5K dataset, the IoU of these vectorized components reached 0.8. The method was neither trained nor tested on rotated floor plans.

One common approach for precise classification with deep learning-based methods is image segmentation. In \cite{Huang.2023} image segmentation and detection (YOLOX-based) were jointly trained using the attention based MuraNet to detect windows, walls and doors in floor plans. MuraNet performed better than YOLOv3 and U-Net on the CubiCasa5k dataset. Rotated floorplans were not considered in their research.

In \cite{Ahmed}, symbols such as doors were captured by SURF \cite{SURF} yielding good results with respect to rotation and scale invariance. However, wall extraction  relies on classical computer vision techniques, including erosion and dilation with a 3x3 kernel, which unfortunately lacks scale and rotation invariance.

 Mingxiang Chen et al. introduced the Graph Neural Network (GNN)-based Line Segment Parser (GLSP)  \cite{MingxiangChen.2023}. Their approach leveraged GNNs to predict the class of line segments, such as Door, Wall, or Window.
The authors also proposed a novel embedding technique called Rotated Region of Interest (RRoI) Pooling. This method enables more effective feature extraction for rotated lines by considering their rotational variations in contrast to traditional Region of Interest (RoI) pooling. However, rotation invariance was not explicitly addressed.

Scanned documents are often slightly rotated, scaled and noisy. This issue was tackled by \cite{Khade.2021} by introducing a geometric feature-based approach for floor plan image retrieval that aims to be rotation and scale invariant. It achieved good results on floor plans from the ROBIN dataset that were rotated by $\pm 5$°. 




\subsection{Rotation invariance through Zernike moments}

\cite{JaeyoungSong.2021} leveraged Graph Neural Networks (GNNs) for indoor element classification in floor plans. Their method involved constructing a Region Adjacency Graph (RAG) with Zernike moments as the node attributes.
They experimented with various GNN architectures and found that GraphSAGE and their novel Distance-Weighted Graph Neural Network (DWGNN) performed best. Although they claimed rotation invariance, their experimental results were limited to $90$° rotated floorplans.
On a newly labeled subset of CubiCasa5k, \cite{JaeyoungSong.2021}'s approach achieved promising F1 scores for indoor element classification. These findings demonstrate the potential of GNNsand Zernike moments for rotation invariant indoor classification.
Zernike moments are rotation invariant and hence are well suited as features.

Zernike polynoms were introduced by F. Zernike in \cite{Zernike}. Using these polynoms, Zernike moments were defined in image analysis via the general theory of moments. 
In \cite{Teh.1988} various moments were examined and compared regarding their sensitivity to image noise, information redundancy and capability for image representation. Zernike and pseudo-Zernike moments outperformed the other moments in these important aspects of image recognition. This leads to Zernike moments now being used widely in image processing due to their many favourable qualities including rotation invariance.

In \cite{Xiang.2012}, Xiang et al. investigated the practical and theoretical rotation, scale and translation invariance of Zernike moments in image processing.
To achieve this invariance, they proposed a normalization step to be applied prior to calculating Zernike moments. Experimental results demonstrated that this approach retained the invariance of the amplitude of Zernike moments when images were loosely rotated.

\section{Methods}
\label{sec: methods}
Our method builds upon Jaeyoung Song et al.'s \cite{JaeyoungSong.2021} GNN framework, enhancing their work by modifying the feature extraction to improve rotation invariance and embedding it within an end-to-end workflow.

We begin with pre-processing, removing additional information from floor plan images. The pre-processed image is used to create the RAG.
Next, a GNN is applied for the classification of the RAG nodes.
During post-processing the detected indoor elements are converted into usable information for, e.g., room-door connectivity and 3D reconstruction.

\begin{figure}[t]
  \centering
  \includegraphics[width=\textwidth]{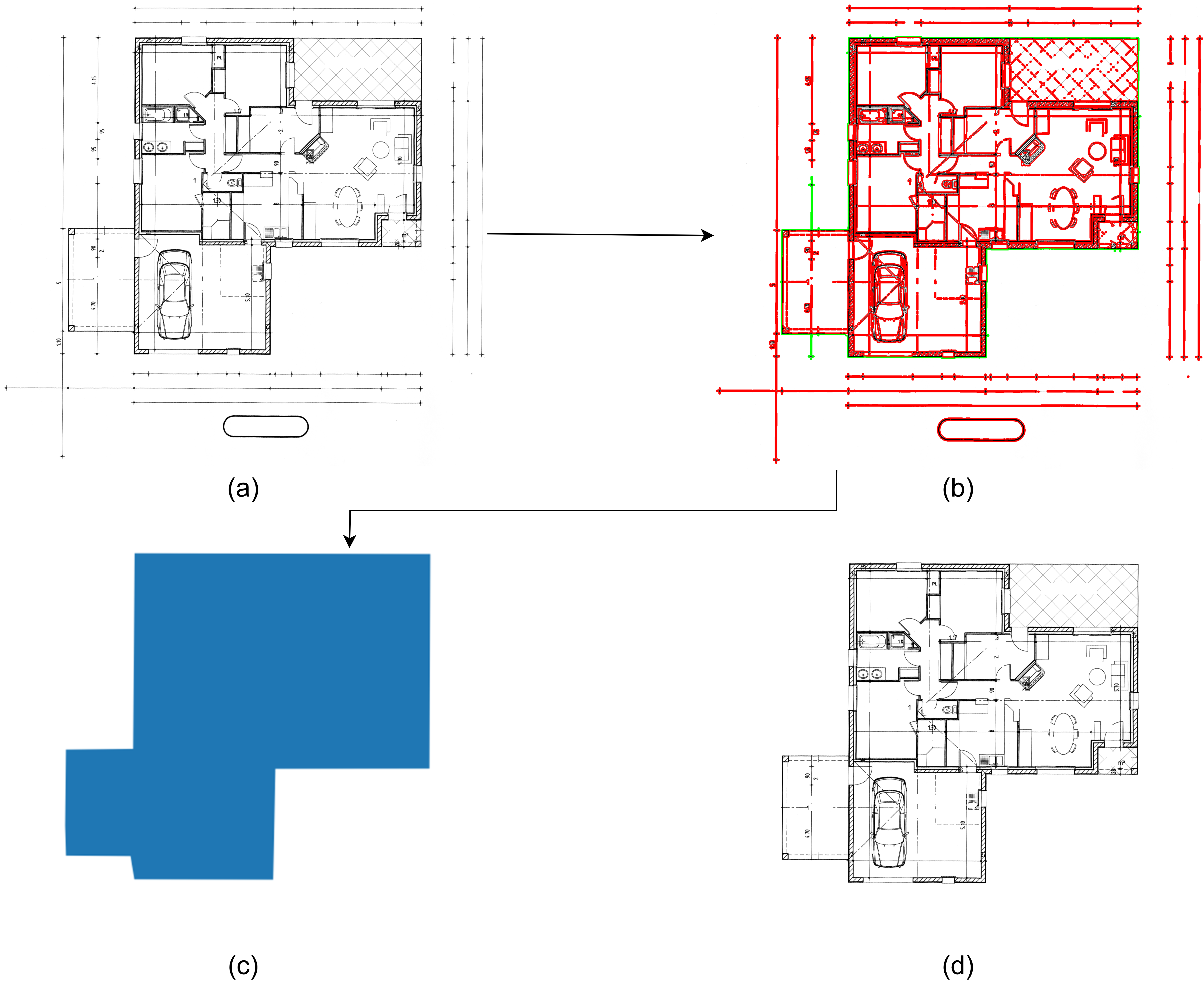}
  \caption{\label{fig:PreProcessing} Filter of building (a) Input image after text removal and dilation (b) red the detected contours and green the largest contour (c) refined polygon of largest contour (d) Filtered building}
\end{figure}
\subsection{Pre-processing}\label{pre-processing}

Floor plans often contain excessive information that can overwhelm the GNN, including text, measurement lines, and legends. To filter out the necessary information, we follow a series of steps.

First, the image is converted into a binary format using a threshold, making the image black and white where the white pixels are the background. Next, EasyOCR is applied to detect the text in the image with bounding boxes that are then coloured white.

To improve contour detection, we remove potential noise and holes by applying a 3x3 dilation kernel, which refines the contour edges. We utilize OpenCV's border-following algorithm to detect the contours in the image Fig. \ref{fig:PreProcessing}(a).

After detecting the contours, we select the largest one (see Fig. \ref{fig:PreProcessing}(b)) and interpret it as an exterior of a polygon. To refine this polygon, we first subtract a circle of radius 5 and than add the circle with radius 5 using the Minkowski sum (later referred to as debuff and buff). This process effectively removes small gaps and ensures that the building's outline is represented accurately as shown in \ref{fig:PreProcessing}(c).

Finally, we colour the complement of the polygon white and the resulting image contains only the filtered building Fig. \ref{fig:PreProcessing}(d).

\subsection{RAG generation and normalization}\label{RAG}
\begin{figure}[tb]
  \centering
  \includegraphics[width=.47\textwidth]{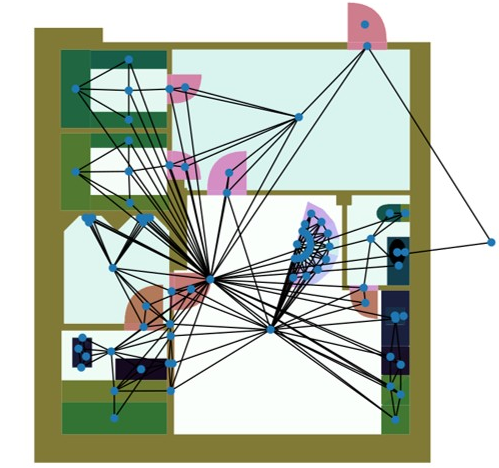}
  \caption{\label{fig:RAG}  RAG of Floorplan. Every polygon has a unique color and is represented by a blue node inside the graph. The node is at the center of mass of the polygon. Two nodes are connected if the corresponding polygons are adjacent.}
\end{figure}
\begin{figure}[tb]
  \centering
  \includegraphics[width=\textwidth]{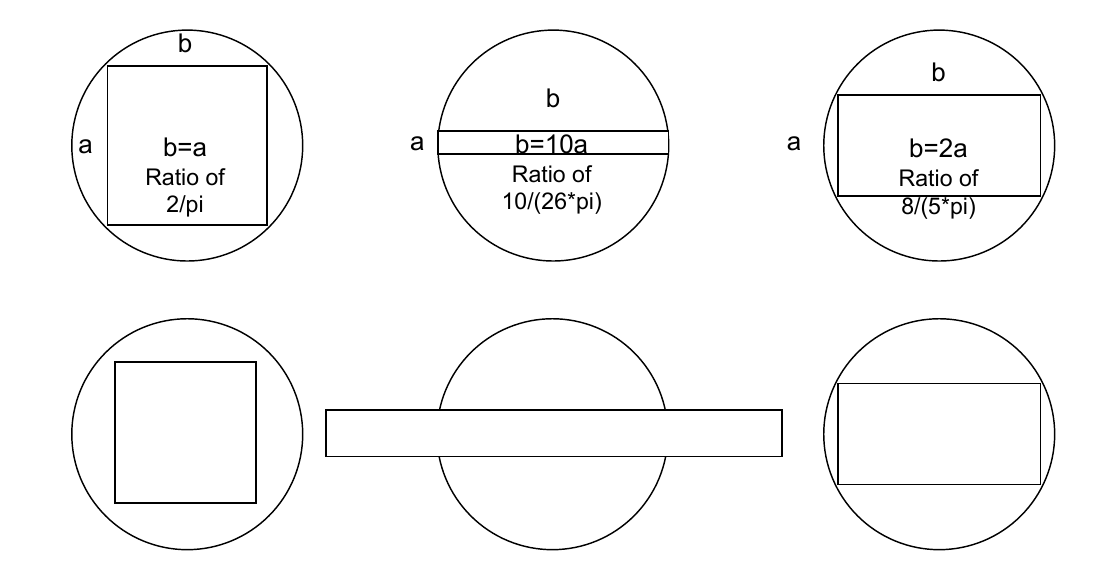}
  \caption{\label{fig:lemma}  Top: Rectangles with different sides a, b. 
  Bottom: The Polygons are scaled with $F_P$. The polygons with invariant ratio greater than $c=0.5$ are inside the circle with radius r after scaling.}
\end{figure}   
The image is vectorized, where each vector (further referred to as polygon) becomes a node in the RAG.
In accordance with \cite{JaeyoungSong.2021}, for every node and its corresponding polygon, we extract the connectivity of the node within the RAG (see Fig. \ref{fig:RAG}) and determine the area of the corresponding polygon. Additionally, we extract the amplitude of Zernike moments up to order $n_{max}\in\bbbn$ with non-negative repetition $m\geq 0$.

Furthermore, we apply the following normalization steps to the polygons to ensure invariance.
 For each polygon $P$ from the vectorization, we centralize $P$ such that the centroid of $P$ is at the origin of the image to assure translation invariance as demonstrated in \cite{Xiang.2012}.
The method from \cite{Xiang.2012} is applied to achieve scale invariance, i.e., each polygon $P$ with area $A_P$ is scaled with a factor $F_P=\sqrt{\frac{A}{A_P}}$, such that $P'=P\cdot F_P$. This scaling process maintains the centroids position at the origin of the image, while ensuring that every polygon $P'$ has the same area $A\in\bbbr^+$.

However, since we do not apply a uniform scaling factor to the entire image but use a unique scaling factor $F_P$ for each polygon $P$, we must ensure that the Zernike moments of the scaled polygon $P'$ are computed correctly. This requires that $P'$ lies within the circle with radius $r$, on which the Zernike moment is calculated.

Let $A=cr^2\pi$. By Lemma \ref{lemma2}, $P'$ lies within the circle with radius $r$ if and only if $c\leq \frac{A_P}{R^2_P\pi}$, where $R_P$ is the radius of the smallest circle centered at the origin containing $P$, as defined in Definition \ref{def}.
The selection of $c$ is a trade-off, as selecting $c$ small weakens the requirement for the whole area to be captured by $C_r$, but also risks down scaling and information loss in practice.

In Fig. \ref{fig:lemma} there is an illustration of the invariant ratio in form of rectangular polygons with sides $a,b$ with different ratios, left to right: a=b with invariant ratio $\frac{2}{\pi}\thickapprox 0.64$,
  a=2b with invariant ratio $\frac{8}{5\pi}\thickapprox 0.51$ and
  a=10b with invariant ratio $\frac{10}{26\pi}\thickapprox 0.12$.
   The polygons are scaled with  $F_P = \sqrt{\frac{cr^2\cdot\pi}{A_P}},\: c=0.5$ where $A_P$ is the area of the polygon to scale and r the radius of the circle. We can see that polygons with an invariant ratio greater than $c=0.5$ are completely captured inside the circle with radius $r$, hence the Zernike moments of the whole polygon would be captured.
\begin{definition}\label{def}
Let $M\subset \bbbr^n$ closed. We define the radius of the smallest circle centered at the origin that contains $M$ as 
$$R_M \coloneqq min\{r \in \bbbr| \lVert v\rVert\leq r\: \forall v \in M\}.$$
Let $M\subset\bbbr^n$ and $c\in\bbbr$ with $c\cdot M=cM\coloneqq \{cv|\: v\in M \}$.\\
$\lambda$ is the Lebesgue measure.
\end{definition}

\begin{lemma}\label{lemma1}
Let $r,F\in\bbbr^+$ and $M\subset \bbbr^n$ closed,  Lebesgue measurable with $\lambda(M)>0$. It holds that  $$F\cdot M\subset C_r\iff F\leq\frac{r}{R_M}. $$
\end{lemma}
\begin{proof}
Note that $\lambda(M)>0\Rightarrow R_M>0$.
\begin{align*}
    F\leq \frac{r}{R_M}&\iff \frac{r}{F}\geq R_M\\
    \overset{Definition\: \ref{def}}&{\iff} \lVert v\rVert \leq \frac{r}{F}\quad\forall v\in M\\
     &\iff \lVert Fv\rVert=F\lVert v\rVert \leq r \quad\forall v\in M\\
     &\iff F\cdot v\in C_r\quad\forall v\in M
     \iff F\cdot M\subset C_R
\end{align*}\qed
\end{proof}
\begin{lemma}\label{lemma2}
Let $c\in\bbbr^+,\: M\subset I \subset\bbbr^2$ closed and Lebesgue measurable with $\lambda(M)>0$. For $F=\sqrt{\frac{c\cdot r^2\pi}{\lambda(M)}}$ it holds that
$$F\cdot M\subset C_r \iff c\leq  \frac{\lambda(M)}{R_M^2\pi}.$$
\end{lemma}
\begin{proof}
\begin{align*}
F\cdot M\subset C_r \overset{Lemma\: \ref{lemma1}}&{\iff} \sqrt{\frac{c\cdot r^2\pi}{\lambda(M)}} = F \leq \frac{r}{R_M}\\
&\iff \frac{\pi c}{\lambda(M)}\leq \frac{1}{R_M^2}\iff c\leq  \frac{\lambda(M)}{R_M^2\pi}
\end{align*}\qed
\end{proof}
Note: Choosing $c\cdot r^2\pi = A$ we have
$$\lambda(F\cdot M)=\lambda(M)\cdot F^2=\lambda(M)\frac{cr^2\pi}{\lambda(M)}=A$$
according to Lebesgue measure properties in $\bbbr^2$.

\subsection{Node classification}
The RAG is fed into a GNN which predicts the labels of the nodes of the RAG. Details can be found in \cite{JaeyoungSong.2021}. This results in a labeled graph consisting of the classes room, wall, door, window, stair, object, porch and outer space.

\subsection{Post-processing}
We apply several post-processing steps to refine the labeled RAG to be applicable for different use cases.

\subsubsection{Room Connectivity}
\begin{figure}[tb]
  \centering
  \includegraphics[width=\textwidth]{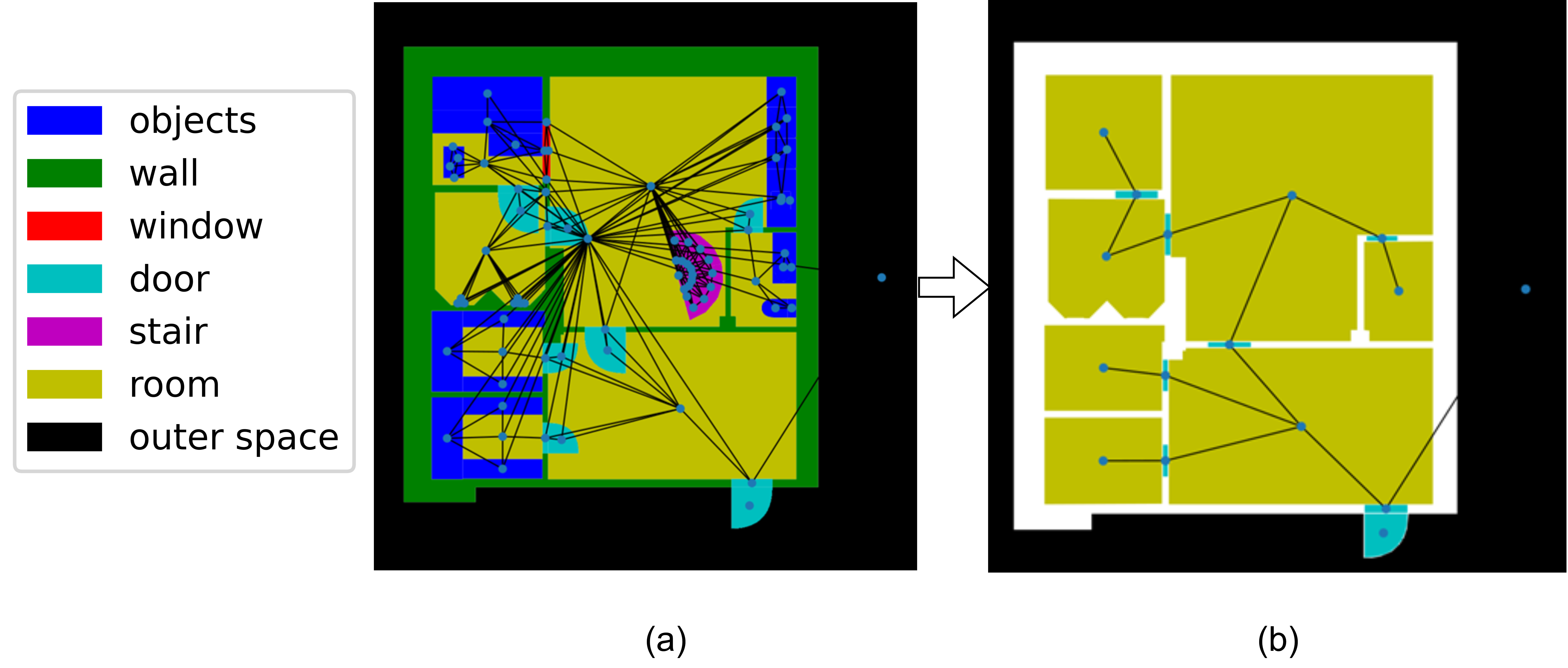}
  \caption{\label{fig:room}  (a) RAG with labels (b) Room connectivity Graph}
\end{figure}
\subsubsection{Door Splitting}\label{DoorSplitting}
For some applications, such as pedestrian flow simulations or floor plan recommendation systems, further knowledge of the relations between rooms is required. This information can be stored in a room connectivity graph, which contains the complete rooms, doors, and outer space polygons as its nodes. Rooms should be connected to doors if a room can be entered and exited through that door and the same holds for the outer space. In Fig. \ref{fig:room} the input RAG and the created room connectivity graph is shown.

First, we recreate the room from labeled polygons. As some objects, such as toilets, showers etc., are part of a room, we need to merge these objects with the room they are located in. The same applies to stairs and door swing areas (see Door Splitting).

Iterative objects, stairs, and door swings polygons are merged into room polygons, if they are connected in the RAG. We then update the new room boundaries and repeat these steps until no changes occur. Note that the object, stair and door swing nodes do not have to be deleted, they can be used as information associated with the room.

Next, we approximate the new room polygons using the Douglas-Peucker algorithm, which helps straighten diagonal lines that were pixel precise before. We also buff and debuff the room polygons to remove small holes inside them, which result from noise, imprecise vectorization or merging. The outer wall is created from the interior of the outer space polygon.

Usually, doors in floor plans are displayed with the door swing area. The door swing area is part of the room, while the other part is enclosed inside the wall. If two door polygons are neighboring in the RAG, we label the larger door polygon as the door swing area and the smaller part as the door part that is embedded within the wall.

\subsubsection{Wall Splitting}
\begin{figure}[tb]
  \centering
  \includegraphics[width=\textwidth]{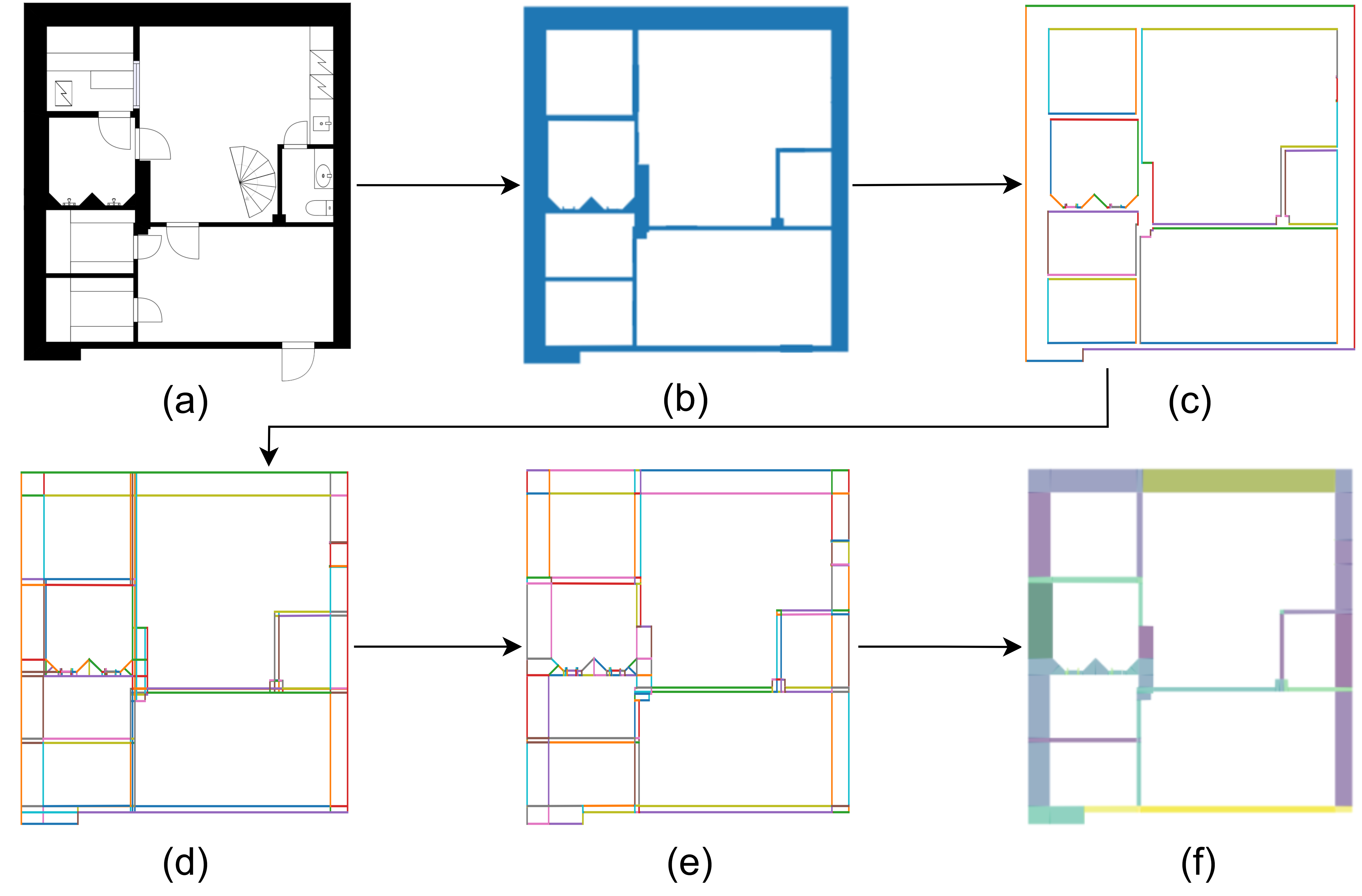}
  \caption{\label{fig:wall splitting}Demonstration from input image to splitted walls  (a) Input image (b) Wall polygon (c) Interior and exterior linear rings of polygon (d) Added separation lines (e) Crossing lines removed (f) polygon creation of remaining lines}
\end{figure}
\begin{algorithm}[t]
\caption{Separation lines}\label{alg:Seperation lines}
\begin{algorithmic}[1]
\State Let $P_{wall}$ be the multi polygon of all wall polygons and linestrings be the set of line strings from the interior and exterior of the polygons in $P_{walls}$.
\For{line\_string in linestrings}\label{sep:linestrings}
    \For{point in line\_string(points)} \label{sep:points}
        \State candidates $\gets$ list()
        \For{line\_string2 in linestrings}\label{sep:lines}
        \For{line in line\_string2(lines)}
            \State shortest\_line $\gets$ shortestline(point, line)\label{sep:shortest}
            \If{shortest\_line $\subset P_{wall}$}\label{sep:filter}
            \If{filter(args)} 
            \State candidates append shortest\_line 
            \EndIf
            \EndIf\label{sep:filter_end}
        \EndFor
        \EndFor\label{sep:lines_end}
        \State sort candidates ascending \label{sep:sort}
        \State best\_candidate $\gets$ candidates[0] 
        \State  separation\_lines append $\gets$ best\_candidate \label{sep:best}
        \For{candidate in candidates}\label{sep:2best}
            \If{angle(candidate, best\_candidate) $> 40$}
            \State  separation\_lines append $\gets$ candidate
            \State break
            \EndIf
        \EndFor \label{sep:2best_end}
    \EndFor
\EndFor \label{sep:linestrings_end}
\State  separation\_lines $\gets$ remove\_crossing\_lines(separation\_lines append)\label{sep:cross}
\State return separation\_lines\label{sep:return}
\end{algorithmic}
\end{algorithm}

First, we want to obtain the entire wall, including windows and doors embedded within it. Doors and windows are contained within a wall. As described in Room Connectivity, we follow the same procedure to merge the door and window nodes into the wall nodes. The resulting wall is then approximated using the Douglas-Peucker algorithm. A resulting wall polygon is shown in Fig. \ref{fig:wall splitting} (b).

To be able to store information about different parts of the wall regarding materials or width, which are useful for several applications, such as the generation of 3D or even Building Information Models (BIM), the wall polygons have to be split in a reasonable manner. To achieve this, we must split the wall into parts associated with individual rooms and load-bearing walls.
To accomplish this, polygons are detected by finding separation lines and constructing the polygon out of these, as implemented in the Algorithms \ref{alg:Seperation lines} and \ref{alg:Polygon construction} and illustrated in Fig. \ref{fig:wall splitting}.

The wall polygons are merged into a multi-polygon $P_{wall}$ (see Fig. \ref{fig:wall splitting} (b)). Each polygon can be defined by its exterior and interior linestrings as shown in Fig. \ref{fig:wall splitting} (c). For every point on each linestring in $P_{wall}$ (Algorithm \ref{alg:Seperation lines}, lines \ref{sep:linestrings}-\ref{sep:linestrings_end}), we search for lines where the polygon should be split. 
For every line contained in the interior or exterior of the polygon (lines \ref{sep:lines}-\ref{sep:lines_end}), the shortest line between this line and point is determined and called $shortest\_line$ (line \ref{sep:shortest}).
If the  $shortest\_line$ is inside the wall polygon and fulfills the $filter$ criteria, the line is a potential separation line and gets appended to $candidates$ (lines \ref{sep:filter}-\ref{sep:filter_end}). The $filter$ function can depend on the use case, we have chosen the $filter$ function to return true if $shortest\_line$ is (almost) orthogonal to $line$ or any of the two lines in $line\_string$ that contains $point$.

The candidates are sorted ascending according to the length of the lines (line \ref{sep:sort}) and the shortest line $best\_candidate$ becomes a separation line (line \ref{sep:best}).
Then, we search for the next longer line inside $candidates$ that has an angle of at least $40°$ to $best\_candidate$. If this line exists it also becomes a separation line (lines \ref{sep:2best}-\ref{sep:2best_end}).
After all the points have completed this procedure Fig. (see \ref{fig:wall splitting} (d)), we proceed with the last steps to remove the larger line from any two lines in $separation\_lines$ that cross and return the remaining $separation\_lines$ (lines \ref{sep:cross}, \ref{sep:return}) (see Fig. \ref{fig:wall splitting} (e)).

The resulting wall parts can easily be associated with their corresponding rooms, using the RAG with updated neighbourhoods.

\subsubsection{Polygon Construction}
\begin{algorithm}[t]
\caption{Polygon construction}\label{alg:Polygon construction}
\begin{algorithmic}[1]
\State lines $\gets$ all lines from linestrings of $P_{wall}$
\While{$|P_{wall}| > 1$}\label{poly:while}
    \State added\_lines $\gets\:$ list()
    \State $P\in P_{wall}:\: dist(P, lines)>\epsilon$\label{poly:point}
    \State sort lines ascending according to distance to point\label{poly:lines}
    \For{line in lines}
        \State is\_crossing $\gets$ false
        \For{added\_line in added\_lines}
            \If{shortest\_line(P, line) intersects added\_line}
                \State is\_crossing $\gets$ true
            \EndIf
        \EndFor
        \If{not is\_crossing}
            added\_lines append(line)
        \EndIf
    \EndFor \label{poly:lines_end}
    \For{line in added\_lines}\label{poly:lines_rem}
        \State is\_crossing $\gets$ false
        \For{added\_line in added\_lines}
        \State mp $\gets$ middle point of line
            \If{shortest\_line(P, mp) intersects added\_line}
                \If{line $\neq$ added\_line}
                    \State is\_crossing $\gets$ true
                \EndIf
            \EndIf
        \EndFor
        \If{ is\_crossing}
            added\_lines remove(line)
        \EndIf
    \EndFor\label{poly:lines_rem_end}
    \State new\_poly $\gets$ convex\_hull(added\_lines)\label{poly:poly}
    \State $P_{wall} \gets P_{wall}-new\_poly$\label{poly:subtract}
\EndWhile\label{poly:while_end}
\end{algorithmic}
\end{algorithm}

Furthermore, these new lines (see Fig. \ref{fig:wall splitting} (e)) have to be assigned correctly to the respective polygons. Therefore, according to Algorithm \ref{alg:Polygon construction}, we select a random point inside $P_{wall}$ (line \ref{poly:point}).
We then iteratively add lines closest to this point, provided that the shortest line between the point $P$ and the $line$ does not intersect with any previously added line (lines \ref{poly:lines}-\ref{poly:lines_end}).

For every $line$ in $added\_lines$, we check that the line between the point $P$ and the midpoint $mp$ of the $line$ does not intersect with any other line in $added\_lines$. If it does intersect, we remove this $line$ from $added\_lines$ (lines \ref{poly:lines_rem}-\ref{poly:lines_rem_end}). 
The convex hull of these lines constitutes the new polygon (line \ref{poly:poly}). This polygon is subtracted from $P_{wall}$ (line \ref{poly:subtract}), and the process is repeated until $P_{wall}$ is empty (lines \ref{poly:while}-\ref{poly:while_end}). An example of the resulting polygons is shown in Fig. \ref{fig:wall splitting} (f).

\section{Experiments}
To evaluate the performance of the modified features, we conducted a series of experiments on different datasets. We provide some details about the datasets used and the implementation of the experiments.
\subsection{Dataset}
For the experiments we used the CubiCasa and CVC datasets.
\subsubsection{CubiCasa}
We used the same dataset as in \cite{JaeyoungSong.2021}, i.e., 400 high-quality floor plans from the CubiCasa5K dataset, which contain different apartment floor plans. The dataset contains SVG formatted floor plan images with vectorized polygons, where a class is assigned to each polygon. These classes are structural elements (walls, windows, doors, and stairs), spatial elements (rooms, porches, and outer space), and objects. When comparing the CubiCasa dataset with the CVC dataset we relabel objects, stairs and porches as rooms because the CVC dataset does not contain these relabeled classes. To the best of our knowledge, this is the only dataset where the vectorized image is completely labeled. 

We split the CubiCasa dataset into fixed training (280), test (80), and validation (40) sets. 
To create a rotated test dataset, we augmented the test set images with a rotation of 45°. To fit the rotated image onto the canvas, the canvas had to be enhanced, leading to higher resolution images. The additional pixels were set to be white. The rotated images were vectorized, and the RAG features were calculated. The newly created polygons were labeled with the label that had the highest IoU value of the rotated labeled polygons.                        
\subsubsection{CVC}
We further used the CVC dataset, which contains 122 scanned floor plan documents divided in 4 different subsets regarding their origin and style. The labels are in SVG format and the classes are structural elements (rooms, walls, doors, windows, parking doors and room separations). We interpret parking doors as doors and ignore the room separation labels. Some areas on the floor plan are not labeled, e.g., stairs.

The floor plans contain additional information, e.g., rooftops, text, measure lines. The dataset was pre-processed using the pre-processing steps described in \ref{pre-processing}.
Due to poor vectorization into the correct indoor elements the data was not suitable for use as training data. The entire 122 floor plans were used as a test dataset.
\subsection{Implementation}
We used the DWGNN with an LSTM aggregator and six layers. We trained for 40 epochs with a learning rate of 0.01 and batch size of 1. We trained on the CubiCasa training dataset with different invariant ratios of $100$, $8^{-1}$ and $80^{-1}$ and one without the proposed normalization.
The features contained 16 Zernike moments (up to order 6). The evaluation occurred before the post-processing steps.
Zernike moments were calculated with the functions provided in \cite{mahotas}.
The hardware characteristics used for the experiments were an Intel Xeon Platinum 8260 CPU, an Nvidia Quadro RTX 8000 GPU and 192 GB of RAM.

\section{Results}

\begin{table}[tb]\centering
    \caption{Comparison of different invariant ratios on different testing data sets. The model was trained on the CubiCasa dataset.}
    \begin{tabular}{*{11}{|l}|}
    \hline
    \multicolumn{2}{|c}{}&\multicolumn{3}{|c}{CubiCasa}& \multicolumn{3}{|c|}{CubiCasa rotated}&
    \multicolumn{3}{c|}{CVC} \\
    \hline
    \multicolumn{2}{|c|}{invariant ratio} & 
    300&$3^{-1}$&$80^{-1}$&300&$3^{-1}$&$80^{-1}$&300&$3^{-1}$&$80^{-1}$  \\ 
    \hline \hline
    \multirow{5}{*}{F1}
    & wall          &\textbf{97.96}  &95.06  &94.6   &82.62  &90.11  &\textbf{92.24}    & 33.98 &\textbf{45.67}  &43.08\\ 
    &window         & \textbf{98.32} &97.46  &91.18  &76.79  &\textbf{79.41}  &78.19    &7.17   &3.72   &\textbf{12.69}\\ 
    &door           & 93.24 &94.95  &\textbf{95.24}  &79.71  &\textbf{87.65}  &83.61    &2.14   &1.87   &3.04\\ 
    &room           & 88.98 &93.62  &\textbf{95.37}  &67.34  &91.84  &\textbf{96.53}    &51.86  &76.83  &\textbf{78.51}\\ 
    &outer space    & 87.58 &93.98  &\textbf{96.57}  &87.08  &98.33  &\textbf{99.61}    &81.18  &94.01  &\textbf{94.81}\\ 
    \hline
    \multirow{5}{*}{IoU}
    & wall          &\textbf{95.99}  &90.58  &89.76  &70.39  &82.01  &\textbf{86.6}     &20.47  &29.59  &27.46\\ 
    &window         & \textbf{96.7}  &95.05  &83.78  &62.32  &\textbf{65.84}  &64.2     &3.72   &1.90   &\textbf{6.78}\\ 
    &door           & 87.34 &90.39  &\textbf{90.9}   &66.27  &\textbf{78.01}  &71.83    &1.08   &\textbf{9.40}   &1.54\\ 
    &room           & 80.14 &88.01  &\textbf{91.15}  &50.76  &84.91  &\textbf{93.29}    &35.01  &62.38  &\textbf{64.62}\\ 
    &outer space    & 77.91 &88.65  &\textbf{93.37}  &77.11  &96.72  &\textbf{99.22}    &68.32  &88.70  &\textbf{90.14}\\ 
    \hline
    \multicolumn{2}{|l|}{Average F1} & 94.62  &\textbf{95.27}  &94.1 &76.61  &87.25  &\textbf{87.64}    &23.79  &32.02  &\textbf{34.33}\\
    \hline
    \end{tabular}
    \label{tab:ratio}
\end{table}

\begin{table}[tb]\centering
    \centering
    \caption{Comparison of indoor classification with and without the proposed normalization steps on the CubiCasa rotated dataset. The model was trained on the CubiCasa data set.}
\begin{tabular}{|l|ll|ll|}
\hline
            & \multicolumn{2}{l|}{$80^{-1}$}                       & \multicolumn{2}{l|}{no normalization}                \\ \cline{2-5} 
            & \multicolumn{1}{l|}{F1}             & IoU            & \multicolumn{1}{l|}{F1}             & IoU            \\ \hline
objects     & \multicolumn{1}{l|}{\textbf{61.10}} & \textbf{43.98} & \multicolumn{1}{l|}{42.61}          & 27.08          \\ \hline
wall        & \multicolumn{1}{l|}{\textbf{92.18}} & \textbf{85.50} & \multicolumn{1}{l|}{88.72}          & 79.78          \\ \hline
window      & \multicolumn{1}{l|}{\textbf{78.25}} & \textbf{64.72} & \multicolumn{1}{l|}{24.82}          & 14.17          \\ \hline
door        & \multicolumn{1}{l|}{\textbf{83.44}} & \textbf{71.59} & \multicolumn{1}{l|}{36.81}          & 22.56          \\ \hline
stair       & \multicolumn{1}{l|}{13.32}          & 7.13           & \multicolumn{1}{l|}{\textbf{34.43}} & \textbf{20.80} \\ \hline
room        & \multicolumn{1}{l|}{\textbf{94.54}} & \textbf{89.64} & \multicolumn{1}{l|}{87.48}          & 77.74          \\ \hline
porch       & \multicolumn{1}{l|}{\textbf{46.24}} & \textbf{30.07} & \multicolumn{1}{l|}{40.02}          & 25.02          \\ \hline
outer space & \multicolumn{1}{l|}{\textbf{99.61}} & \textbf{99.22} & \multicolumn{1}{l|}{99.15}          & 98.32          \\ \hline
Average     & \multicolumn{1}{l|}{\textbf{67.01}} & \textbf{56.09} & \multicolumn{1}{l|}{50.7}           & 38.16          \\ \hline
\end{tabular}
\label{tab:paper}
\end{table}
In Table \ref{tab:ratio} and \ref{tab:paper} the average was taken over all classes except the outer space class.
In Table \ref{tab:ratio} we see the comparison of classification with different invariant ratios for the normalization step on different datasets.
The average F1 scores on the CubiCasa dataset range from $94.1\%$ to $95.27\%$ having very similar performance with different invariant ratios. The wall and window detection in the CubiCasa dataset with invariant ratio $300$ outperformed the other invariant ratios with an excellent F1 score ($97.96\%$, $98.32\%$) and IoU ($95.99\%$, $96.7\%$). On the other hand, the class detection with an invariant ratio of $3^{-1}$ provided the most consistent F1 scores and IoU's across all classes, performing well on every class detection task and achieving the highest average F1 score ($95.27\%$).

The differences between invariant ratios become more apparent in the rotated CubiCasa and CVC datasets. On both datasets, classification with an invariant ratio of $300$ yielded the lowest average F1 scores and IoU values. 
For the rotated CubiCasa dataset, the results were not as good as those for the non-rotated CubiCasa dataset, but using an invariant ratio of $3^{-1}$ or $80^{-1}$ still showed promising performance as can be seen in the average F1 score ($87.25\%$, $87.64\%$). The F1 score and IoU for the window and door classes were highest when using the invariant ratio of $3^{-1}$ (window: $79.41\%$, $65.84\%$, door: $87.65\%$, $78.01\%$). Otherwise, the invariant ratio of $80^{-1}$ performed best with an average F1 score of $87.64\%$.

The results on the CVC dataset were less impressive overall.
Only the room and outer space classification exhibited reasonable performance when using an invariant ratio of $3^{-1}$ or $80^{-1}$. The highest F1 score for wall detection was achieved with an invariant ratio of $3^{-1}$, resulting in a value of $45.67\%$, while the highest IoU was $29.59\%$. In contrast, the window and door classification showed F1 scores and IoU values in the single digits.
The classification with invariant ratio $80^{-1}$ had the best average F1 score of $34.33\%$.

In Table \ref{tab:paper}, we compare the adjusted calculation of Zernike moments as described in Section \ref{RAG} with an invariant ratio of $80^{-1}$ versus without normalization steps. The average F1 score and average IoU when normalizing the polygons are significantly higher at $67.01\%$ and $56.09\%$, respectively, compared to $50.7\%$ and $38.16\%$ when leaving out the normalization steps. The wall, door, room, and outer space classes exhibit high F1 scores and IoU values when using an invariant ratio of $80^{-1}$. However, the object porch and stair detection show relatively low scores. When not normalizing only the stair class achieved a higher F1 score and IoU than when normalizing. In all other classes, our proposed method improved upon the original scores, particularly in window and door detection, where the scores more than double when normalizing.

\section{Discussion}
We added a normalization step in the feature extraction and established an invariant ratio to improve the Zernike moment calculation.
The normalization steps clearly improved rotation invariance as seen in the results on the rotated CubiCasa dataset. Interestingly, stair nodes were not classified correctly after rotation, which may be due to similarity with object nodes. Overall, the method was capable of classifying most indoor elements correctly despite never seeing rotated floor plans.
Furthermore, the performance on the rotated CubiCasa dataset improved clearly with smaller invariant ratios.

The generalization to an unseen dataset also improved with smaller invariant ratios but did not perform as well as on the CubiCasa dataset. 
Recall that a smaller invariant ratio corresponds to a smaller scaling factor, meaning more shape information is captured by Zernike moments. It's interesting to see that better generalization occurs with more visual information, as visuals differ between datasets. The overall poor generalization to the CVC dataset is likely due to the more complex RAG in the CVC dataset, caused by additional information such as rooftops, as they are depicted by schematic lines splitting room polygons. Also, some nodes in the RAG are represented by polygons containing areas of room and door, making labeling tasks ambiguous, e.g., doors with no enclosed area. This is due to the vectorization step only vectorizing connected areas. Text removal and possible symbol removal could be done with inpainting that reconstructs the covered spatial information for a preciser vectorization.

The vectorization process has its drawbacks, but it also offers advantages, including capturing the precision of floor plans with a level of detail that can surpass the provided ground truth.
A notable benefit is that the vectorization can be used to divide the entire floor plan into meaningful segments, enabling the creation of a RAG that includes relevant neighborhoods for real life applications. This has proven particularly useful in post-processing and offers potential for further integration with 3D models.

For floor plans with few additional technical details, e.g., escape and rescue plans, the vectorization may be enhanced with additional object detection or splitting vectors into smaller parts. On more technical drawings like the CVC dataset, a segmentation approach seems more suitable.

Moreover the presented wall splitting algorithm has shown desired wall splitting results on the cubicasa dataset but has not been evaluated with a metric or tested on complex buildings with, e.g., round walls.

\section{Conclusion}
We presented an end-to-end pipeline from input image to a RCG and indoor element classification, ready for use with 3D models. 
Our pipeline consists of four stages: pre-processing the input image to remove text and some additional information, a RAG generation with feature extraction, prediction of indoor element classes of RAG nodes using a graph neural network (GNN), and post-processing steps to extract the RCG and apply a wall splitting algorithm.
As part of our pipeline we introduced a novel rule based wall splitting algorithm that returns walls that can be associated with rooms and are convex.

Furthermore, we enhanced the feature extraction when generating the RAG by utilizing normalization steps and established an invariant ratio that provides a criterion to ensure that a polygon is fully captured when calculating Zernike moments. 
To evaluate the performance of the normalization and the influence of the invariant ratio, we performed experiments on the CubiCasa and CVC dataset, with different invariant ratios and an ablation study of the normalization method.
The performance on the rotated Cubicasa data increased significantly using these normalization steps, and further improvements were achieved by choosing a small enough invariant ratio.



\bibliographystyle{splncs04}
\begin{credits}
    \ackname
     This version of the contribution has been 
    accepted for publication, after peer review but is not the Version of 
    Record and does not reflect post-acceptance improvements, or any corrections. 
    Use of this Accepted Version is subject to the publisher’s Accepted Manuscript terms of use 
    https://www.springernature.com/gp/open-research/policies/accepted-manuscript-terms
    \end{credits}
\bibliography{literature}
\end{document}